%% file: main.tex
\def\year{2021}\relax
\documentclass[letterpaper]{article} 
\usepackage{aaai21}  
\usepackage{times}  
\usepackage{helvet} 
\usepackage{courier}  
\usepackage[hyphens]{url}  
\usepackage{graphicx} 
\usepackage{natbib}  
\usepackage{caption} 
\usepackage[utf8]{inputenc} 
\usepackage[T1]{fontenc}    
\usepackage{url}            
\usepackage{booktabs}       
\usepackage{amsfonts}       
\usepackage{nicefrac}       

\usepackage{microtype}      
\usepackage{multirow}
\usepackage[linesnumbered,ruled]{algorithm2e}
\usepackage[title]{appendix}
\usepackage{amssymb}
\usepackage{csquotes}
\usepackage{bbm}
\usepackage{amsthm}
\input{math_commands}

\theoremstyle{definition}

\usepackage{tablefootnote}
\usepackage{float}
\restylefloat{table}

\newcommand{\mcal}{\mathcal}

\title{How Does Data Augmentation Affect Privacy in Machine Learning?}

\author {

        Da Yu\thanks{The work was done when this author was an intern at Microsoft Research Asia.}\textsuperscript{\rm 1},
        Huishuai Zhang\textsuperscript{\rm 2},
        Wei Chen\textsuperscript{\rm 2},
        Jian Yin\textsuperscript{\rm 1},
        Tie-Yan Liu\textsuperscript{\rm 2}\\
}
\affiliations {
    \textsuperscript{\rm 1} School of Computer Science and Engineering, Sun Yat-sen University. \\Guangdong Key Laboratory of Big Data Analysis and Processing,  Guangzhou 510006, P.R.China\\
    \textsuperscript{\rm 2} Microsoft Research Asia \\
    \{yuda3@mail2, issjyin@mail\}.sysu.edu.cn, 
    \{huishuai.zhang, wche, tie-yan.liu\}@microsoft.com
}

\begin{document}

\maketitle

\begin{abstract}

It is observed in the literature that data augmentation can significantly mitigate membership inference (MI) attack. However, in this work, we challenge this observation by proposing new MI attacks to utilize the information of augmented data.  MI attack  is widely used to measure the model's information leakage of the training set. We establish the optimal membership inference when the model is trained with augmented data, which inspires us to formulate the MI attack  as a set classification problem, i.e., classifying a set of augmented instances instead of a single data point, and design input permutation invariant features. Empirically, we demonstrate that the proposed approach universally outperforms original methods when the model is trained with data augmentation. Even further, we show that the proposed approach  can achieve higher MI attack success rates on models trained with some data augmentation than the existing methods on  models trained without data augmentation. Notably, we achieve 70.1\% MI attack success rate on CIFAR10 against a wide residual network while previous best approach only attains 61.9\%. This suggests the privacy risk of models trained with data augmentation could be largely underestimated.

\end{abstract}

\input{introduction}

\input{background}

\input{main_result}

\input{experiment}

\bibliography{privacy}



\end{document}

%% file: math_commands.tex
\usepackage{amsmath,amsfonts,bm}

\def\eqref#1{equation~\ref{#1}}

\def\1{\bm{1}}

\def\vv{{\bm{v}}}

\def\vx{{\bm{x}}}

\DeclareMathAlphabet{\mathsfit}{\encodingdefault}{\sfdefault}{m}{sl}
\SetMathAlphabet{\mathsfit}{bold}{\encodingdefault}{\sfdefault}{bx}{n}

\newtheorem{definition}{\textbf{Definition}}

\newtheorem{theorem}{\textbf{Theorem}}
\newtheorem{proposition}{\textbf{Proposition}}
\newtheorem{remark}{\textbf{Remark}}

\DeclareMathOperator*{\argmin}{arg\,min}

%% file: introduction.tex
\section{Introduction}

The training process of machine learning model often needs access to private data, e.g., applications in financial and medical fields. Recent works have shown that the trained model  may leak the information of its private training set \citep{fredrikson2015model, wu2016methodology, shokri2017membership, hitaj2017deep}. As the machine learning models are ubiquitously deployed in real-world applications, it is important to quantitatively analyze the information leakage of their training sets.  One fundamental approach reflecting the privacy leakage of a model about its training set is the \emph{membership inference}  \cite{shokri2017membership,yeom2018privacy,salem2018ml,nasr2018machine,long2018understanding,jia2019memguard,song2019privacy, chen2020machine}, i.e., an adversary, who has access to a target model, determines whether a data point is used to train the target model (being a member) or not (not being a member). Membership inference (MI) attack is formulated as a binary classification task. A widely adopted measure for the performance of an MI attack algorithm in literature is the MI success rate over a balanced set that contains half training samples and half test samples. A randomly guessing attack will have success rate of $50\%$ and hence a good MI algorithm should have success rate above $50\%$.

\begin{figure} 
    \centering
  \includegraphics[width=0.75\linewidth]{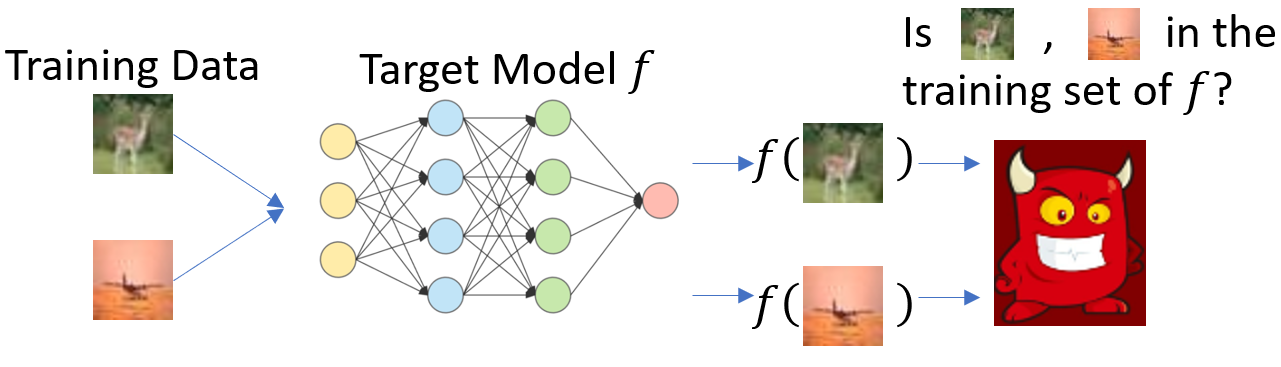}
  \caption{Overview of black-box membership inference in machine learning. The adversary has access to  target model's outputs of given samples. The adversary then infers whether the sample is in the target model's training set or not. Higher inference success rate indicates more severe privacy leakage.}
  \label{fig:MI}
\end{figure}

 It is widely believed that the capability of membership inference is largely attributed to the generalization gap \cite{shokri2017membership,yeom2018privacy,li2020membership}. The larger performance difference of the target model on the training set and on the test set, the easier to determine the membership of a sample with respect to the target model. 
\emph{Data augmentation} is known to be an effective approach to produce well-generalized models.  Indeed, existing MI algorithms obtain significantly lower MI success rate against models trained with data augmentation than those trained without data augmentation \cite{sablayrolles2019white}. It seems that the privacy risk  is largely relieved when data augmentation is used.

 We challenge this belief by elaborately showing how data augmentation affects the MI attack. 
 We first establish the optimal membership inference when the model is trained with data augmentation from the Bayesian perspective.  The optimal membership inference indicates that we should use the set of augmented instances of a given sample rather than a single sample to decide the membership. This matches the intuition because the model is trained to fit the augmented data points instead of a single data point.   We also explore the connection between optimal membership inference and group differential privacy, and obtain an upper bound of the success rate of MI attack. 

In this paper, we focus on the \emph{black-box} membership inference \cite{shokri2017membership,yeom2018privacy,salem2018ml,song2019privacy,sablayrolles2019white}. We give an illustration of black-box MI in Figure~\ref{fig:MI}. The black-box setting naturally arises  in the \emph{machine learning as a service (MLaaS)} system. In MLaaS, a service provider trains a ML model on private crowd-sourced data and releases the model to users through prediction API.  Under the black-box setting, one has access to the model's output of a given sample. Typical outputs are the loss value \cite{yeom2018privacy,sablayrolles2019white} and the predicted logits \cite{shokri2017membership,salem2018ml}.  We use the  loss value of a given sample as it is shown to be better than the logits \cite{sablayrolles2019white}. 

Motivated by the optimal membership inference, we formulate the membership inference as a set classification problem where the set consists of loss values of the augmented instances of a sample evaluated on the target model. We design two new algorithms for the set classification problem. The first algorithm  uses threshold on the average of the loss values of augmented instances, which is inspired by the expression of the optimal membership inference. The second algorithm uses neural network as a membership classifier. For the second algorithm, we show it is important to design features that are invariant to  the permutation on loss values. Extensive experiments demonstrate that our algorithms significantly improve the success rate over existing membership inference algorithms.  We even find that the proposed approaches on models trained with some data augmentation  achieve higher MI attack success rate than the existing methods on the model trained without data augmentation. Notably, our approaches achieve $>70\%$ MI attack success rate against a wide residual network, whose test accuracy on CIFAR10 is more than $95\%$.

Our contributions can be summarized as follows.  First, we establish the optimal membership inference when the model is trained with data augmentation. Second, we formulate the membership inference as a set classification problem and propose two new approaches to conduct membership inference, which achieve significant improvement over existing methods. This suggests  that \emph{the privacy risk of models trained with data augmentation could be largely underestimated}. To the best of our knowledge, this is the first work to systematically study the effect of data augmentation on membership inference and reveal non-trivial theoretical and empirical findings.

\subsection{Related Work}

Recent works have explored the relation between generalization gap and the success rate of membership inference. \citet{shokri2017membership, sablayrolles2019white} empirically observe that better generalization leads to worse inference success rate.  \citet{yeom2018privacy} show the success rates of some simple attacks are directly related to the model's generalization gap. For a given model, \citet{li2020membership} empirically verify the success rate of MI attack is upper bounded by generalization gap. However, whether the target model is trained with data augmentation, the analysis and  algorithms of previous work only use single instance to decide membership. Our work fills this gap by formulating and analyzing membership inference when data augmentation is applied. 

\citet{song2019privacy} show \emph{adversarially robust} \cite{madry2018towards} models are more vulnerable to MI attack. They identify one major reason of this phenomenon is the increased generalization gap caused by adversarial training. They also design empirical attack algorithm which leverages the adversarially perturbed image (this process needs white-box access to the target model). In this paper, we choose perturbations following the common practice of data augmentation, which can reduce the generalization gap and do not need white-box access to the target model. 

Differential privacy \citep{dwork2006calibrating, algofound} controls how a single sample could change the parameter distribution in the worst case. How data augmentation affects the DP guarantee helps us to understand how the data augmentation affects membership inference.  In Section~\ref{sec:dp}, we give a discussion on the relation between data augmentation, differential privacy, and membership inference.

%% file: background.tex
\section{Preliminary}
\label{sec:pre}
We assume that a dataset $D$  consists of samples of the form $(\vx,y)\in \mcal{X}\times\mcal{Y}$, where $\vx$ is the feature and $y$ is the label. A model $f$ is a mapping from feature space to label, i.e., $f: \mcal{X}\rightarrow \mcal{Y}$. We assume that the model is parameterized by $\theta\in\mathbb{R}^{p}$. We further define a loss function $\ell(f(\vx),y)$ which measures the performance of the model on a data point, e.g., the cross-entropy loss for classification task. We may also written the loss function as $\ell(\theta, d)$ for data point $d=(\vx,y)$ in this paper. The learning process is conducted by minimizing the \emph{empirical loss}: $\sum_{d\in D}\ell(\theta, d)$.


The data are often divided into training set $D_{train}$ and test set $D_{test}$ to properly evaluate the model performance on unseen samples. The generalization gap $G$ represents the difference of the model performance between the training set and the test set,
\begin{equation}
\begin{aligned}
\label{eq:generalization}
G=\mathbb{E}_{d\sim D_{test}}[\ell(\theta, d)]-\mathbb{E}_{d\sim D_{train}}[\ell(\theta,d)].
\end{aligned}
\end{equation}


\subsection{Data Augmentation}
Data augmentation is well  known as a good way to improve generalization. It transforms each sample into  similar variants and uses the transformed variants as the training samples.  We use $\mathcal{T}$ to denote the set of all possible transformations. For a given data point $d$, each transformation $t\in \mcal{T}$ generates one augmented instance $t(d)=(\tilde \vx,y)$. For example, if $\vx$ is a natural image, the transformation could be rotation by a specific degree or flip over the horizontal direction.  The set $\mcal{T}$ then contains the transformations with all possible rotation degrees and all directional flips. The size of $\mcal{T}$ may be infinite and we usually only use a subset in practice. Let $T\subset \mcal{T}$ be a subset of transformations. The cardinality of $|T|$ controls the strength of the data augmentation. We use $T(d)=\{t(d);t\in T\}$ and $\ell_{T}(\theta, d)=\{\ell(\theta,\tilde d);\tilde d\in T(d)\}$ to denote the set of augmented instances and corresponding loss values. With data augmentation,  the learning objective is to fit the augmented instances
\begin{equation}
\begin{aligned}
\label{eq:augmented_obj}
\theta=\argmin_{\theta}\sum_{d\in D}\;\sum_{ \tilde d \in T(d)} \ell(\theta,\tilde d).
\end{aligned}
\end{equation}


\subsection{Membership Inference}

Membership inference is a widely used tool to quantitatively analyze the information leakage of a trained model.  Suppose the whole dataset consists of $n$ i.i.d. samples $d_{i},\ldots,d_{n}$ from a data distribution, from which we choose a subset as the training set. We decide membership using $n$ i.i.d. Bernoulli samples $\{m_{1},\ldots,m_{n}\}$ with a positive probability $\mathbb{P}(m_{i}=1)=q$. Sample $d_i$ is used to train the model if $m_{i}=1$ and is not used if  $m_{i}=0$. Given the learned parameters $\theta$ and $d_{i}$, membership inference is to infer $m_{i}$, which amounts to computing $\mathbb{P}(m_{i}=1|\theta,d_{i})$.


That is to say, membership inference aims to find the posterior distribution of $m_{i}$ for given $\theta$ and $d_{i}$. Specifically, \citet{sablayrolles2019white} shows that it is sufficient to use the loss of the target model to determine the membership $m_i$ under some assumption on the the posterior distribution of $\theta$. They predict $m_{i}=1$ if $\ell(\theta,d_{i})$ is smaller than a threshold $\tau$, i.e.
\begin{equation}
\label{eq:m_loss}
\begin{aligned}
M_{loss}(\theta,d_{i})=1\;\;\;\; \text{if} \;\;\;\; \ell(\theta,d_{i})<\tau.
\end{aligned}
\end{equation}
This membership inference is well formulated for the model trained with original samples. However, it is not clear how to conduct membership inference and what is the optimal algorithm  when data augmentation is used in the training process\footnote{\citet{sablayrolles2019white} directly applies the  algorithm (Equation~\ref{eq:m_loss}) for the case with data augmentation.}. We  analyze these questions in next sections.


%% file: main_result.tex
\section{Optimal Membership Inference with Augmented Data}
\label{sec:optimal_mi}


When data augmentation is applied, the process  \[\{d_{i}\} \rightarrow \{T(d_{i})\} \rightarrow \{\theta, m_i\}\] forms a \emph{Markov chain}, which is due to the described learning process. That is to say, given $T(d_{i})$, $d_i$ is independent from $\{\theta, m_i\}$. Hence we have
\[H(m_i| \theta, T(d_i)) = H(m_i|\theta, T(d_i), d_i) \ge H(m_i |\theta, d_i),\]
where  $H(\cdot| \cdot)$ is the conditional entropy \cite{ghahramani2006information}, the first equality is due to the Markov chain and the second inequality is due to the property of conditional entropy.

This indicates that we could get less uncertainty of $m_i$ based on $\{\theta, T(d_i)\}$ than based on $\{\theta, d_i\}$. Based on this observation, we give the following definition.

\begin{definition} (Membership inference with augmented data)
\label{def:definition_mi_aug}
For given parameters $\theta$, data point $d_{i}$ and transformation set $T$, membership inference  computes
\begin{equation}
\begin{aligned}
\mathbb{P}(m_{i}=1|\theta,T(d_{i})).
\end{aligned}
\end{equation}
\end{definition}


For the membership inference with augmented data given by Definition~\ref{def:definition_mi_aug}, we establish an equivalent formula in the Bayesian sense, which sets up the optimal limit that our algorithm can achieve. Without  loss of generality, suppose we want to infer $m_{1}$. Let $\mathcal{K}=\{m_{2},\ldots,m_{n},T(d_{2}),\ldots,T(d_{n})\}$  be the status of remaining data points. Theorem~\ref{lma:optimal_mi} provides the Bayesian optimal membership inference rate.




\begin{theorem}
\label{lma:optimal_mi}
The optimal membership inference for given $\theta$ and $T(d_{1})$ is $\mathbb{P}(m_{1}=1|\theta,T(d_{1}))=$
\[\mathbb{E}_{\mathcal{K}}\left[\sigma\left(\log\left(\frac{\mathbb{P}(\theta|m_{1}=1,T(d_{1}),\mcal{K})}{\mathbb{P}(\theta|m_{1}=0,T(d_{1}),\mcal{K})}\right)+\log\left(\frac{q}{1-q}\right)\right)\right],\]
where $\sigma(x):=(1+e^{-x})^{-1}$ is the sigmoid function and $q :=\mathbb{P}(m_{1}=1)$ is a constant.
\end{theorem}

\begin{proof}
Apply the law of total expectation and Bayes' theorem, we have
\begin{equation}
\begin{aligned}
\label{eq:pro_lma1_0}
&\mathbb{P}(m_{1}=1|\theta,T(d_{1}))=\mathbb{E}_{\mathcal{K}}[\mathbb{P}(m_{1}=1|\theta,T(d_{1}),\mathcal{K})]\\
&=\mathbb{E}_{\mathcal{K}}\left[\frac{\mathbb{P}(\theta|m_{1}=1,T(d_{1}),\mcal{K})\mathbb{P}(m_{1}=1)}{\mathbb{P}(\theta|T(d_{1}),\mcal{K})}\right].
\end{aligned}
\end{equation}
Substitute  $q:=\mathbb{P}(m_{i}=1)$ and let
{\small
\begin{equation}
\begin{aligned}
\label{eq:pro_lma1_defineab}
\alpha :=\mathbb{P}(\theta|m_{1}=1,T(d_{1}),\mcal{K}),\;\beta :=\mathbb{P}(\theta|m_{1}=0,T(d_{1}),\mcal{K}).
\end{aligned}
\end{equation}
}
Notice that $\mathbb{P}(\theta|T(d_{1}),\mcal{K})=q\alpha+(1-q)\beta$. Then rearranging Eq~(\ref{eq:pro_lma1_0}) gives
\begin{equation}
\begin{aligned}
\label{eq:pro_lma1_1}
\mathbb{P}(m_{1}=1|\theta,T(d_{1}))=\mathbb{E}_{\mathcal{K}}\left[\left(1+(\frac{1-q}{q})\frac{\beta}{\alpha}\right)^{-1}\right],
\end{aligned}
\end{equation}
which concludes the proof.
\end{proof}

We note that the expression in Theorem~\ref{lma:optimal_mi} measures how a single data point affects the parameter posterior in expectation. This is connected with the \emph{differential privacy} \cite{dwork2006calibrating, dwork2006our}, which measures how a single data point affects the parameter posterior in the worst case.   We give a discussion on the relation between data augmentation, differential privacy, and membership inference in Section~\ref{sec:dp}.

\section{Membership Inference with Augmented Data Under a Posterior Assumption}
\label{sec:mi_bayesian}

 \begin{figure} 
    \centering
  \includegraphics[width=0.5\linewidth]{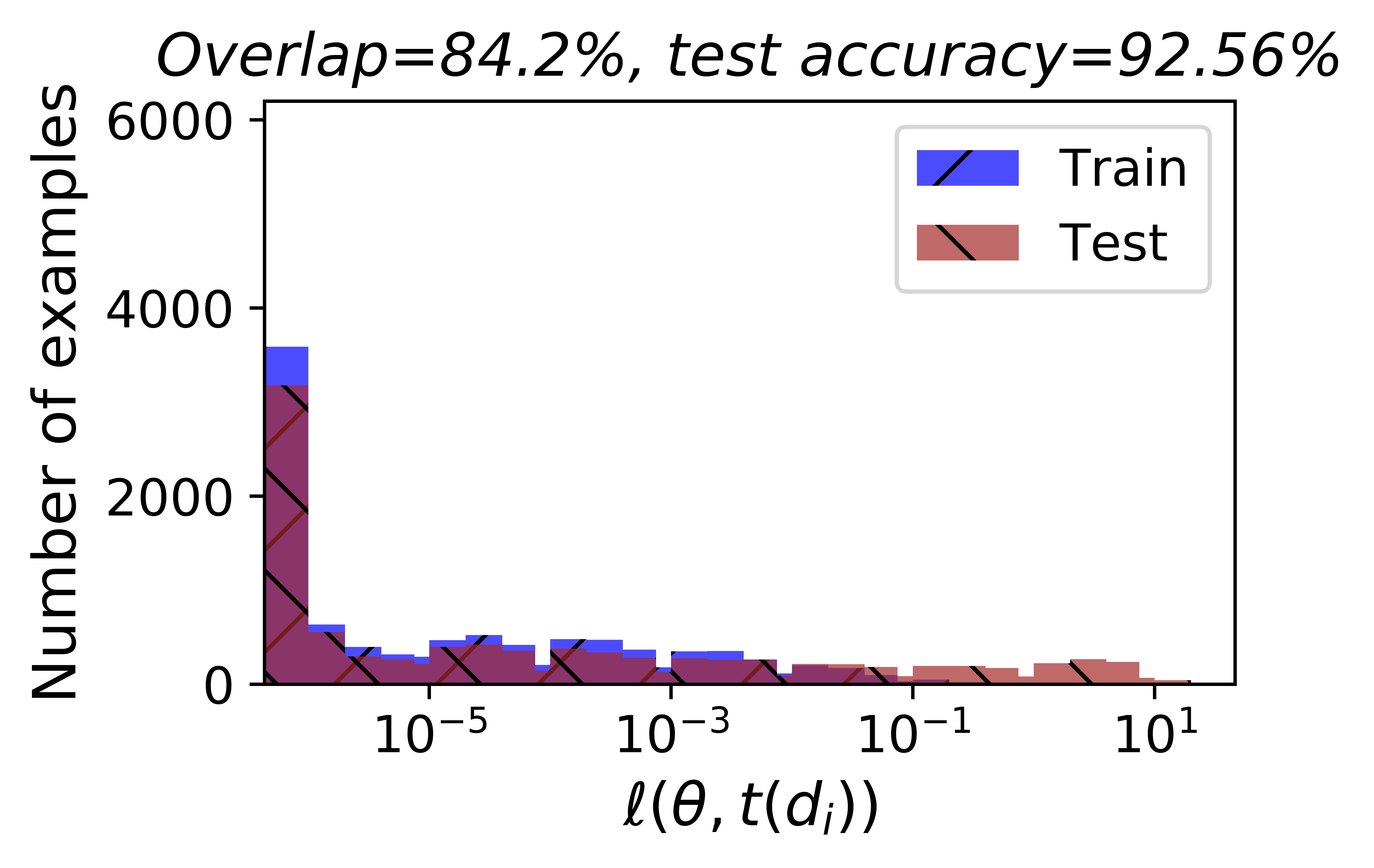}
  \caption{Distribution of single loss values on CIFAR10 dataset. The model is ResNet110 trained with $|T|=10$. The plot uses 10000 examples from training set and 10000 examples from test set. The dark region is the overlap area between training and test distributions. The membership of a value inside overlap region is hard to decide.  }
  \label{fig:loss_hist}
\end{figure}

In this section we first show the optimal membership inference explicitly depends on the loss values of augmented examples when $\theta$ follows a \emph{posterior} distribution. Then we give a membership inference algorithm based on our theory.

\subsection{Optimal Membership Inference Under a Posterior Assumption}

In order to further explicate the optimal membership inference (Theorem \ref{lma:optimal_mi}), we need knowledge on the probability density function of $\theta$. Following the wisdom of energy based model \cite{lecun2006tutorial,du2019implicit}, we assume the posterior distribution has the form,

\begin{equation}
\begin{aligned}
\label{eq:assumption}
p(\theta|m_{1},T(d_{1}),\mcal{K})\propto \exp\left(-\frac{1}{\gamma}L(\theta)\right),
\end{aligned}
\end{equation}
where  $L(\theta)=\sum_{i=1}^{n}m_{i}\sum \ell_{T}(\theta,d_{i})\geq 0$ is the objective to be optimized and $\gamma$ is the temperature parameter. We note that Eq~(\ref{eq:assumption}) meets the intuition that the parameters with lower loss on training set have larger chance to appear after training. Let $p_{\mcal{K}}(\theta)=\frac{\exp(-\frac{1}{\gamma}\sum_{i=2}^{n}m_{i}\sum \ell_{T}(\theta,d_{i}))}{\int_{z}\exp(-\frac{1}{\gamma}\sum_{i=2}^{n}m_{i}\sum \ell_{T}(z,d_{i}))dz}$  be the PDF of $\theta$ given $\mcal{K}$. The denominator is a constant keeping $\int_{z}p_{\mcal{K}}(z)dz=1$. Theorem~\ref{thm:opt_bayesian} present the optimal algorithm under this assumption.


\begin{theorem}
\label{thm:opt_bayesian}
Given parameters $\theta$ and $T(d_{1})$, the optimal membership inference is 
\[\mathbb{P}\left(m_{1}=1|\theta,T(d_{1})\right)=\mathbb{E}_{\mcal{K}}\left[\sigma\left(\tau-\frac{1}{\gamma}\sum \ell_{T}(\theta, d_{1})+c_{q}\right)\right],\]
where $\tau:=- \log\left(\int_{z}\exp(-\frac{1}{\gamma}\sum \ell_{T}(z, d_{1}))p_{\mcal{K}}(z) dz\right)$, $c_{q}:=\log(q/(1-q))$ and $\sigma(\cdot)$ is the sigmoid function.
\end{theorem}

\begin{proof}
 For the $\alpha$ and $\beta$ defined in Eq~(\ref{eq:pro_lma1_defineab}), we have
\begin{equation}
\begin{aligned}
\label{eq:thm0}
\alpha &= \frac{e^{-(1/\gamma)\sum \ell_{T}(\theta,d_{1})}e^{-(1/\gamma)\sum_{i=2}^{n}m_{i}\sum \ell_{T}(\theta,d_{i})}}{\int_{z}e^{-(1/\gamma)\sum \ell_{T}(z,d_{1})}e^{-(1/\gamma)\sum_{i=2}^{n}m_{i}\sum \ell_{T}(z,d_{i})}dz}\\
& =\frac{e^{-(1/\gamma)\sum \ell_{T}(\theta,d_{1})} p_{\mcal{K}}(\theta)}{\int_{z}e^{-(1/\gamma)\sum \ell_{T}(z,d_{1})}p_{\mcal{K}}(z) dz}
\end{aligned}
\end{equation}
and $\beta=p_{\mcal{K}}(\theta)$. Therefore, we have $\log(\frac{\alpha}{\beta}) =$
\begin{equation}
\begin{aligned}
\label{eq:thm1}
 -\frac{1}{\gamma}\sum  \ell_{T}(\theta,d_{1}) - \log\left(\int_{z}e^{-(1/\gamma)\sum \ell_{T}(z,d_{1})}p_{\mcal{K}}(z) dz\right).
\end{aligned}
\end{equation}
Then plugging  Eq~(\ref{eq:thm1}) into Theorem~\ref{lma:optimal_mi} yields Theorem~\ref{thm:opt_bayesian}.
\end{proof}

The $\tau$ in Theorem~\ref{thm:opt_bayesian} represents the magnitude of  $\ell_{T}(d_{1})$ on parameters trained without $T(d_{1})$. Smaller $\sum \ell_{T}(\theta,d_{1})$  indicates higher $\mathbb{P}(m_{1}=1)$.  This motivates us to design a membership inference algorithm based on a threshold on loss values (see Algorithm \ref{alg:mean}). Data points with loss values smaller than such a threshold are more likely to be training data.

A second observation is that the optimal membership inference \emph{explicitly} depends on the set of loss values. Therefore, membership inference attacks against the model trained with data augmentation are ought to leverage the loss values of all augmented instances for a given sample. We give more empirical evidence in Section~\ref{subsec:mi_mean}.

\subsection{Inference Algorithm in Practice}
\label{subsec:mi_mean}

 Inspired by Theorem~\ref{thm:opt_bayesian},  we predict the membership by comparing $\frac{1}{k}\sum \ell_{T}(\theta,d_{i})$  with a given threshold. The pseudocode is presented in Algorithm~\ref{alg:mean}.

\begin{algorithm}
	\caption{Membership inference with average loss values ($M_{mean}$).}
	\label{alg:mean}
	\SetKwInOut{Input}{Input}
	
	\SetKwInOut{Output}{Output}

	\Input{Set of loss values $\ell_{T}(\theta,d )$,  threshold $\tau$.}
	\Output{Boolean value, $true$ denotes $d $ is a member.}
	Compute $v=mean(\ell_{T})$.

	
	
	Return $v<\tau$.

\end{algorithm}

 We can set threshold $\tau$ in Algorithm~\ref{alg:mean} based on the outputs of shadow models or tune it based on validation data as done in previous work \cite{sablayrolles2019white,song2019privacy}. Though simple, Algorithm~\ref{alg:mean}  significantly outperforms $M_{loss}$ by a large margin. The experiment results can be found in Section~\ref{sec:exp}. 
 
 We now give some empirical evidence on why $M_{mean}$ is better than $M_{loss}$.  We plot the bar chart of single loss values in Figure~\ref{fig:loss_hist} (we random sample one loss value for each example). We train the ResNet110 model \cite{he2016deep} to fit CIFAR10 dataset\footnote{\url{https://www.cs.toronto.edu/~kriz/cifar.html}.}.  We use the same transformation pool $\mcal{T}$ as \citet{he2016deep} which contains horizontal flipping and random clipping. As shown in Figure~\ref{fig:loss_hist},  the overlap area of the loss values between the training samples and the test samples is large when data augmentation is used. For the value inside the overlap area, it is impossible for $M_{loss}$ to classify its membership confidently. Therefore, the overlap area sets up a limit on  the success rate of $M_{loss}$.

 Next, we plot the distribution of $\frac{1}{k}\sum \ell_{T}(\theta,d_{i})$ in Figure~\ref{fig:mean_std_hist}.  The overlap area in Figure~\ref{fig:mean_std_hist} is significantly smaller compared to Figure~\ref{fig:loss_hist}. This indicates  classifying the mean of $\ell_{T}(\theta,d_{i})$ is easier than classifying a single loss value.

\begin{figure} 
    \centering
  \includegraphics[width=0.5\linewidth]{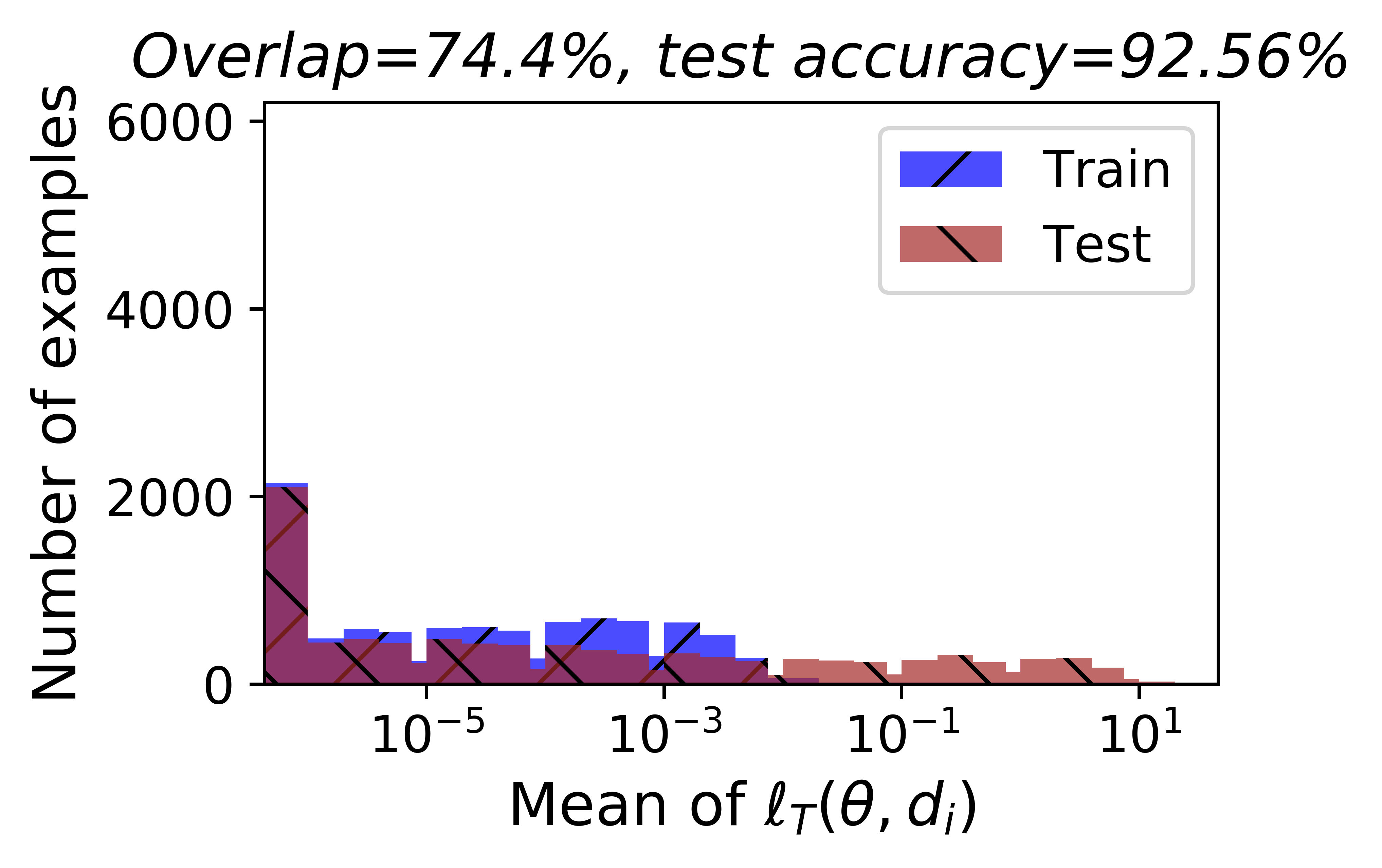}
  \caption{Distribution of the mean  of $\ell_{T}(\theta,d_{i})$. The experiment setting is the same as Figure~\ref{fig:loss_hist}. When using mean  as metric, the overlap area between training and test distributions is smaller than using single loss, which indicates that $\frac{1}{k}\sum \ell_{T}(\theta,d_{i})$ is a better feature.}
  \label{fig:mean_std_hist}
\end{figure}

\section{Membership Inference with Augmented Data Using Neural Network}

\label{sec:mi_nn}

We have shown that the mean of loss values is the optimal membership inference when $\theta$ follows a posterior assumption and demonstrate its good empirical performance. However, if in practice $\theta$  does not exactly follow the posterior assumption, it is possible to design features to incorporate more information than the average of loss values to boost the membership inference success rate.   In this section, we use more features in $\ell_T(\theta, d)$ as input and train a neural network $\mcal{N}$ to do the membership inference. The general algorithm is presented in Algorithm~\ref{alg:NN}.

   \begin{algorithm}
	\caption{Membership inference with neural network.}
	\label{alg:NN}
	\SetKwInOut{Input}{Input}
	
	\SetKwInOut{Output}{Output}

	\Input{Set of loss values of a target  sample $\ell_{T}(\theta,d )$; 
    MI network $\mcal{N}$ and hyperparameters $\mcal{H}$;
    some raw data $\mcal{S}:=\{(\ell_T(\theta, \hat{d}), \mathbf{1}_{\hat{d}\in D_{train}}\}$.}
	\Output{boolean value, $true$ denotes $d $ is a member.}
    
    Build input feature vectors $\vv$ from $\ell_{T}(\theta,\hat{d})$ and construct a training set $\mcal{S}':=\{(\vv, \mathbf{1}_{\hat{d}\in D_{train}})\}$;
    
    
 Use the training set $\mcal{S}'$ and hyperparameters $\mcal{H}$ to train MI network  $\mcal{N}$;

	Return $\mcal{N}(\ell_{T}(\theta,d ))$.

\end{algorithm}

In Algorithm~\ref{alg:NN}, each record in  raw data $\mcal{S}$ consists of the loss values of a given example and corresponding membership. The training data of MI network is built from $\mcal{S}$. Specifically, the loss values of each record are transformed into the input feature vector  of MI network $\mcal{N}$.

 Then  the key point is to design input feature of the network $\mcal{N}$. We first  use the raw values in $\ell_{T}(\theta,d )$ as features. We show this  solution has poor performance because it is not robust to the permutation on loss values. Then we design permutation invariant features through the  \emph{raw moments} of $\ell_{T}(\theta,d )$ and demonstrate its superior performance.

\subsection{A Bad Solution}

A straightforward implementation is to train a neural network as a classifier whose inputs are the loss values of all the augmented instances for a target sample. The pseudocode of this implementation is presented in Algorithm~\ref{alg:losses}. We refer to this approach as $M_{NN\_loss}$.  Surprisingly, the  success rate of $M_{NN\_loss}$ is much worse than $M_{mean}$  though $M_{NN\_loss}$  has access to more information. 

\begin{algorithm}
	\caption{Generating input features from raw losses.}
	\label{alg:losses}
	\SetKwInOut{Input}{Input}
	
	\SetKwInOut{Output}{Output}

	\Input{Set of loss values  $\ell_T(\theta, \hat{d})$.}
	\Output{Feature vector $\vv$.}

    	 Concatenate the elements in $\ell_{T}$ into vector $\vv$. 

	 Return $\vv$.

\end{algorithm}

\begin{figure}
    \centering
  \includegraphics[width=0.5\linewidth]{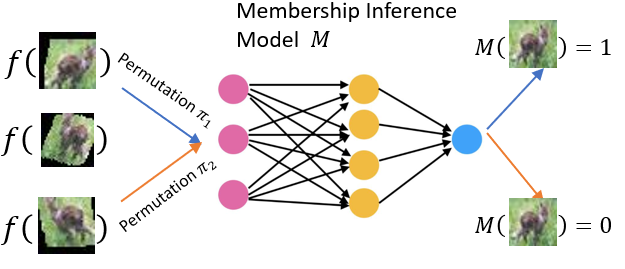}
  \caption{Neural network is not robust to permutation of input features. Changing the order of features will change the prediction. However, the order of augmented instances is not relevant to the membership. }
  \label{fig:per_noinv}
\end{figure}

We note that different from standard classification task, the order of elements in set $\ell_{T}(\theta,d )=\{\ell(\theta,\tilde d);\tilde d\in T(d )\}$ should not  affect the decision of the MI classifier because of the nature of the problem.  However, the usual neural network is not invariant  to the permutation on input features. For a neuron with non-trivial weights, changing the positions of input features would change its output. We illustrate this phenomenon in Figure~\ref{fig:per_noinv}: the order of elements in $\ell_{T}$, which is not relevant to the target sample's membership, however has large influence on the output of network.

\subsection{Building Permutation Invariant Features }

Inspired by the failure of $M_{NN\_loss}$, we design features that are invariant to the permutation on $\ell_{T}(\theta,d)$. We first define functions whose outputs are permutation invariant with respect to their inputs. Then we use permutation invariant functions to encode the loss values into  permutation invariant features.

Recall  $k=|T|$ is the number of augmented instances for each sample.  Let $a\in\mathbb{R}^{k}$ be a vector version of $\ell_{T}(\theta,d)$. Let $\pi\in\Pi$ be a permutation of $a$ and $P_{\pi}\in \mathbb{R}^{k\times k}$ be its corresponding permutation matrix. The following definition states a transformation function satisfying the permutation invariant property.

\begin{definition}
\label{def:invariant}
A function $f:\mathbb{R}^{k}\rightarrow\mathbb{R}^{p}$ is permutation invariant if for  arbitrary $\pi_{i},\pi_{j}\in\Pi$ and $a\in\mathbb{R}^{k}$:
\[f(P_{\pi_{i}}a)=f(P_{\pi_{j}}a).\]
\end{definition}

Clearly, the \emph{mean} function in  Algorithm~\ref{alg:mean} satisfies Definition~\ref{def:invariant}. However, using the $mean$ to encode $\ell_{T}(\theta,d)$ may introduce too much information loss.

To better preserve the information, we turn to the raw moments of $\ell_{T}(\theta,d)$. The $i_{th}$ raw moment $v_{i}$ of a probability density (mass) function $p(z)$ can be computed as   $v_{i}=\int_{-\infty}^{+\infty}z^{i}p(z)dz$.
The moments of $\ell_{T}(\theta,d)$ can be computed easily because $\ell_{T}(\theta,d)$ is a valid empirical distribution with uniform probability mass. Shuffling the loss values would not change the moments. More importantly, for probability distributions in bounded intervals, the moments of all orders  uniquely determines the distribution (known as \emph{Hausdorff moment problem \cite{shohat1943problem}}). The pseudocode of generating permutation invariant features through raw moments  is in Algorithm~\ref{alg:features}.

\begin{algorithm}
	\caption{Generating permutation invariant  features through raw moments.}
	\label{alg:features}
	\SetKwInOut{Input}{Input}
	
	\SetKwInOut{Output}{Output}

	\Input{Set of loss values $\ell_{T}(\theta,d)$; the highest order of moments $m$.}
	\Output{Permutation invariant features $\vv$}

	 \For{$i\in [m]$}{
     	 Compute the normalized $i_{th}$ raw moment: $v_{i}:=\left(\frac{1}{|T|}\sum_{l\in \ell_{T}(\theta,d)}l^{i}\right)^{1/i}$,}
	 
	 
     Concatenate $\{v_i; i\in[m]\}$ into a vector $\vv$. 
     
	 Return $\vv$.

\end{algorithm}

We note that any classifier using the features generated by Algorithm~\ref{alg:features} is permutation invariant with respect to $\ell_{T}(\theta,d)$.  We then use Algorithm~\ref{alg:features} to construct $\mcal{S}^{'}$ in Algorithm~\ref{alg:NN}. This approach is referred to as $M_{moments}$. In our experiments, $M_{moments}$ achieves the highest inference success rate. Experiments details and results can be found in Section~\ref{sec:exp}.

 \begin{figure}
    \centering
  \includegraphics[width=1.0\linewidth]{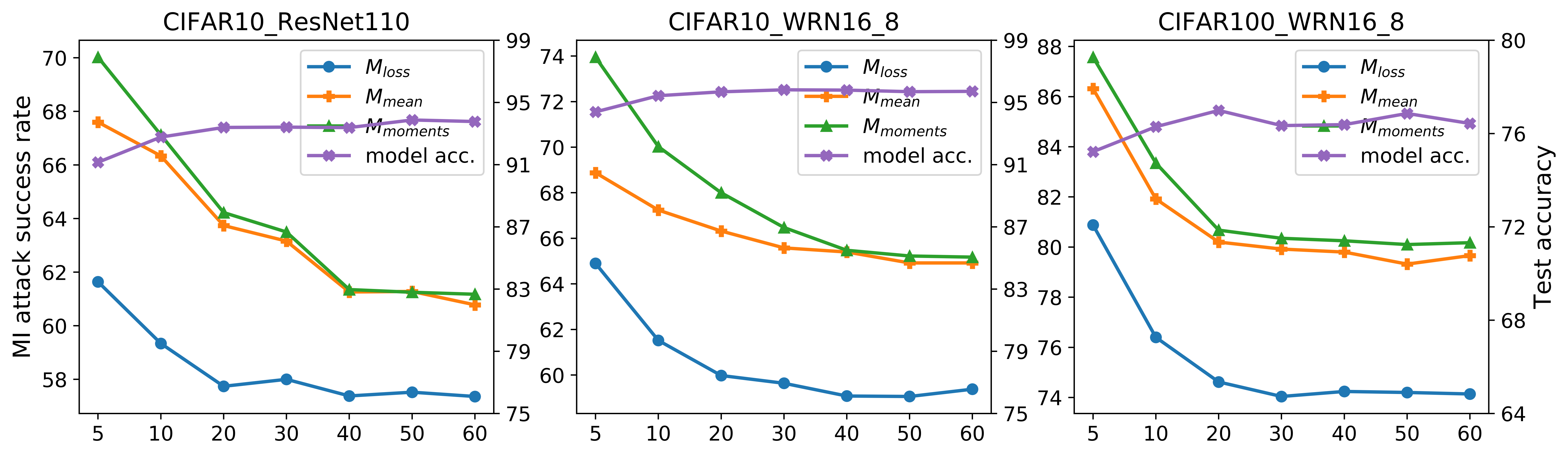}
  \caption{Membership inference success rates with varying $k$ on CIFAR10 and CIFAR100. The left y-axis denotes the membership inference attack success rate. The right y-axis denotes the test accuracy of target models. Our algorithms achieve universally better performance on different datasets and models with varying choices of $k$. }
  \label{fig:varying_N}
\end{figure}

%% file: experiment.tex
\section{Experiments}

\label{sec:exp}

In this section, we empirically compare the proposed inference algorithms with state-of-the-art membership inference attack algorithm, and demonstrate that the proposed algorithms achieve superior performance over different datasets, models and choices of data augmentation. 
\begin{table*}
\small
\renewcommand{\arraystretch}{1.}
\centering
    \caption{Membership inference success rates (in $\%$). We report top-1 test accuracy for CIFAR10 and top-5 accuracy for ImageNet.  The numbers under  algorithm name are the attack success rates. When $k=0$, we run the proposed methods with 10 randomly augmented instances as input anyway. The baseline attack $M_{loss}$ is introduced in Section~\ref{sec:pre}. The row with $k=0$ denotes the model is trained without data augmentation. Test accuracy denotes the target model's classification accuracy on test set.  }
    \label{tbl:brief_attack_success_rates}
        \begin{tabular}{llllllll}
        \hline
        \hline
            Model                           &  Dataset  &     $|T|$                        &  Test accuracy            & $M_{loss}$ & $M_{NN\_loss}$ &  $M_{mean}$  & $M_{moments}$\\
            \hline
            \multirow{2}{*}{2-layer ConvNet}& CIFAR10  &   \multirow{1}{*}{$k=0$}& 59.7     & 83.7     & 83.6     &   83.7     &  83.7 \\\cline{2-8}
                                            & CIFAR10   &\multirow{1}{*}{$ k=3$} & 64.6    &  82.2    & 85.7     &  90.3      & \textbf{91.3}       \\\hline

            \multirow{2}{*}{ResNet110}      & CIFAR10   & \multirow{1}{*}{$k=0$}&  84.9          & 65.4      & 65.4      &   65.4        &  65.6 \\\cline{2-8}

                                        & CIFAR10  &\multirow{1}{*}{$ k=10$} & 92.7           &  58.8    & 61.8    &  66.3       & \textbf{67.1}      \\\hline
            \multirow{2}{*}{WRN16-8}    &  CIFAR10  &\multirow{1}{*}{$k=0$}&  89.7          & 62.9      & 62.8      &   62.8        &  62.9 \\\cline{2-8}

                                        &  CIFAR10  &  \multirow{1}{*}{$ k=10$} & 95.2       &  61.9     & 63.1   &  68.9      & \textbf{70.1}      \\\hline
            \multirow{1}{*}{ResNet101}    
                                        &  ImageNet  &  \multirow{1}{*}{$ k=10$} & 93.9       &  68.3     & 68.9   &  73.9      & \textbf{75.2}      \\\hline
                                        \hline
        \end{tabular}

\end{table*}

We first introduce the datasets and target models with the details of experiment setup. Our source code is publicly available \footnote{\url{https://github.com/dayu11/MI_with_DA}}. 

\subsubsection{Datasets} We use benchmark datasets for image classification: CIFAR10, CIFAR100, and ImageNet1000. CIFAR10 and CIFAR100 both have 60000 examples including 50000 training samples and 10000 test samples. CIFAR10 and CIFAR100 have 10 and 100 classes, respectively. ImageNet1000 contains more than one million high-resolution images with 1000 classes. We use the training and validation sets provided by ILSVRC2012\footnote{\url{http://image-net.org/challenges/LSVRC/2012/}.}.

\subsubsection{Details of used data augmentation}  We consider $6$ standard transformations in image processing literature, including flipping, cropping, rotation, translation, shearing, and cutout \citet{devries2017improved}.

For each $t\in\mcal{T}$, the operations are applied with a random order and  each operation is conducted with a randomly chosen parameter (e.g. random rotation degrees). Following the common practice, we sample different  transformations for different training samples. 

\subsubsection{Target models} We choose target models with varying capacity, including a small convolution model used in previous work \cite{shokri2017membership,sablayrolles2019white}, deep ResNet \cite{he2016deep} and wide ResNet \cite{zagoruyko2016wide}.  The small convolution model  contains $2$ convolution layers with $64$ kernels, a global pooling layer and a fully connected layer of size $128$. The small model is trained for $200$ epochs with initial learning rate 0.01. We decay the learning rate by $10$ at the 100-th epoch. Following \citet{shokri2017membership, sablayrolles2019white}, we  randomly choose $15000$ samples as training set for the small model. The ResNet models for CIFAR is a  deep ResNet model with 110 layers and  a wide ResNet model WRN16-8. The detailed configurations and training recipes for deep/wide ResNets can be found in the original papers.  For ImageNet1000, we use the ResNet101 model and follow the training recipe in \citet{sablayrolles2019white}.

\subsubsection{Implementation details of membership inference algorithms}

All the augmented instances are randomly generated. We use $k$ to denote the number of augmented instances for one image. The number of augmented images is the same for training target models and conducting membership inference attacks. The benchmark algorithm is $M_{loss}$, which achieves the state-of-the-art black-box membership inference success rate \cite{sablayrolles2019white}.   For  $M_{loss}$, \emph{we report the best result among using every element in $\ell_{T}(\theta,d)$ and the loss of original image}. We tune the threshold of  $M_{loss}$ and $M_{mean}$ on valid data following previous work \cite{sablayrolles2019white, song2019privacy}.  For $M_{NN\_loss}$ and $M_{moments}$, we use $200$ samples from the training set of target model and $200$ samples from the test set to build the training data of inference network.   The inference network has two hidden layers with $20$ neurons and Tanh non-linearity as activation function.   We randomly choose $2500$ samples from the training set of target model and $2500$ samples from the test set to evaluate the inference success rate. The samples used to evaluate inference success rate have no overlap with inference model's training data.  Other details of implementation can be found in our submitted code.

\subsubsection{Experiment Results}

We first present the inference success rate with a single $k$.  We use $k=10$ as default. For 2-layer ConvNet, we choose $k=3$ because its small capacity.  The results are presented in Table~\ref{tbl:brief_attack_success_rates}.

When data augmentation is used, algorithms using  $\ell_{T}(\theta,d)$ universally outperform $M_{loss}$. Algorithm~\ref{alg:losses} has inferior inference success rate compared to $M_{mean}$ and $M_{moments}$ because it is not robust to permutation on input features. The best inference success rate is achieved by $M_{moments}$, which utilizes the most information while being invariant to the permutation on $\ell_{T}(\theta,d)$.

Remarkably, when $k=10$, $M_{moments}$ has inference success rate higher than $70\%$ against WRN16-8, whose top-1 test accuracy on CIFAR10 is more than $95\%$!  Moreover, in Table~\ref{tbl:brief_attack_success_rates}, our algorithm on models trained with data augmentation obtains higher inference success rate than previous algorithm ($M_{loss}$) on models trained without data augmentation. We note that the generalization gap of models with data augmentation is much smaller than that of models without data augmentation. \emph{This observation challenges the common belief that models with better generalization provides better privacy. }

We further plot the inference success rates of $M_{loss}$, $M_{mean}$ and $M_{moments}$ with varying $k$ in Figure~\ref{fig:varying_N}.  For all algorithms, the inference success rate gradually degenerates as $k$ becomes large. Nonetheless, our algorithms  consistently outperform $M_{loss}$ by a large margin for all  $k$.

\section{Connection with Differential Privacy}
\label{sec:dp}

Differential privacy (DP) measures how a single data point affects the parameter posterior in the worst case. In this section, we show an algorithm with DP guarantee can provide an upper bound on the membership inference. DP is  defined for a random algorithm $\mcal{A}$ applying on two datasets $D$ and  $D'$ that differ from each other in one sample, denoted as $D\sim^{1}D'$. Differential privacy ensures the change of arbitrary instance does not significantly change the algorithm's output.

\begin{definition} ($\epsilon$ - differential privacy \cite{dwork2006calibrating})
A randomized learning algorithm $\mcal{A}$ is $\epsilon$-differentially private with respect to $D$ if for any subset of possible outcome $S$ we have $\max_{D\sim^{1}D^{'}} \frac{\mathbb{P}(\mcal{A}=S|D)}{\mathbb{P}(\mcal{A}=S|D^{'})}\leq e^{\epsilon}.$
\end{definition}

However, in the formula of Theorem~\ref{lma:optimal_mi}, the change/removal of one sample $d_1$ indicates change/removal of a  set of training instances $T(d_1)$.  We need \emph{group differential privacy} to give upper bound on the quantity of Theorem~\ref{lma:optimal_mi}.

Let $D$ be a training set with $n$ samples and $D\sim^{k}D'$ denote that two datasets differ in $k$ instances. Group different privacy and differential privacy are connected via the following property.

\begin{remark} (Group differential privacy) \label{rmk:gdp}
 If $\mathcal{A}$ is $\epsilon$-differentially private with respect to $D$, then it is also $k\epsilon$-group differentially private  for the group size $k$.
\end{remark}

  Let $D_{aug}=\{T(d_{i});m_{i}=1,i\in [n]\}$ be the augmented training set with $k$ transformations, i.e., $|T(d_i)| = k $. For mean query based algorithms (e.g. gradient descent algorithm), the sensitivity of any instance is reduced to $\frac{1}{k}$. Therefore, a learning algorithm $\mathcal{A}$ that is $\epsilon$-differentially private with respect to dataset $D$ is $\frac{\epsilon}{k}$-differentially private with respect to $D_{aug}$\footnote{The $\frac{\epsilon}{k}$-DP is at instance level, i.e. $D_{aug}\sim^1 D_{aug}'$.}.  With this observation, we have an upper bound on the optimal membership inference in Theorem~\ref{lma:optimal_mi}. 

\begin{proposition}
\label{pro:upper_bound}
If the learning algorithm is $\frac{\epsilon}{k}$-differentially private with respect to $D_{aug}$, we have
\[\mathbb{P}(m_{1}=1|\theta,T(d_{1}))\leq \sigma\left(\epsilon+\log(q/(1-q))\right).\]
\end{proposition}

\begin{proof}
For any given $\mcal{K}$, we have 
\begin{flalign}
&\frac{\mathbb{P}(\theta|m_{1}=1,T(d_{1}),\mcal{K})}{\mathbb{P}(\theta|m_{1}=0,T(d_{1}),\mcal{K})} \leq \max_{D_{aug}\sim^{k}D^{'}_{aug}}\frac{\mathbb{P}(\mcal{A}=S|D_{aug})}{\mathbb{P}(\mcal{A}=S|D^{'}_{aug})}\nonumber\\
&\leq e^{\epsilon}. \label{eq:upper_bound_dp}
\end{flalign}

The first inequality is due to the definitions of $T(d_1)$ and group differential privacy, and the second inequality is due to the property of group differential privacy (Remark~\ref{rmk:gdp}).  Substituting Eq~(\ref{eq:upper_bound_dp}) into Theorem~\ref{lma:optimal_mi} yields the desired bound.
\end{proof}

Proposition~\ref{pro:upper_bound} tells that if the learning algorithm is $\frac{\epsilon}{k}$-DP with respect to $D_{aug}$ , which is true for differentially private gradient descent \cite{bassily2014private}, the upper bound of the optimal membership inference is not affected by the number of transformations $k$. This is in contrast with previous membership inference algorithm that only considers single instance \cite{sablayrolles2019white}, i.e., formulated as $\mathbb{P}(m_{1}=1|\theta, \tilde d_{1})$, 
where $\tilde d_{1}$ can be any element in $T(d_{1})$. Due to the result in \citet{sablayrolles2019white},  the upper bound of $\mathbb{P}(m_{1}=1|\theta, \tilde d_{1})$ scales with $\frac{\epsilon}{k}$ for mean query based algorithms, which monotonically decreases with $k$. This suggests  the algorithm in \citet{sablayrolles2019white} has limited performance especially when $k$ is large. 


\section{Conclusion}
\label{sec:conclusion}
In this paper, we revisit the influence of data augmentation on the privacy risk of machine learning models. We show the optimal membership inference in this case explicitly depends on the augmented dataset (Theorem~\ref{lma:optimal_mi}). When the posterior distribution of parameters follows the Bayesian posterior, we give an explicit expression of the optimal membership inference (Theorem~\ref{thm:opt_bayesian}). Our theoretical analysis inspires us to design  practical attack algorithms. Our algorithms achieve state-of-the-art membership inference success rates against well-generalized models, suggesting that the privacy risk of existing deep learning models may be largely underestimated.  An important future research direction is to mitigate the privacy risk incurred by data augmentation.

\section*{Acknowledgments}

Da Yu and Jian Yin are supported by the National Natural Science Foundation of China (U1711262, U1711261,U1811264,U1811261,U1911203,U2001211), Guangdong Basic and Applied Basic Research Foundation (2019B1515130001),  Key R\&D Program of Guangdong Province (2018B010107005). Huishuai Zhang and Jian Yin are corresponding authors.